%% file: main.tex
\documentclass{article}

\usepackage{arxiv}
\usepackage{amsmath}
\usepackage{array}
\usepackage{multirow}
\usepackage[utf8]{inputenc}
\usepackage[T1]{fontenc}
\usepackage{url}
\usepackage{float}
\usepackage{booktabs}
\usepackage{amsfonts}
\usepackage{siunitx}
\usepackage{comment}
\usepackage{lipsum}
\usepackage{graphicx}
\usepackage{subfigure}
\usepackage{natbib}
\usepackage{doi}
\usepackage{macros}
\usepackage{algorithm}
\usepackage{algpseudocode}
\usepackage{amsmath, amssymb}

\DeclareMathOperator{\reduce}{reduce}
\DeclareMathOperator{\supp}{supp}

\newtheorem{proposition}{Proposition}[section]

\title{ROTE: Benchmarking Neural Memorization on Complexity-Controlled Symbolic Sequences}

\author{ \href{https://orcid.org/0000-0003-1778-393X}{\includegraphics[scale=0.06]{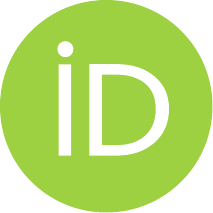}\hspace{1mm}Xinye Chen}\thanks{Authors in alphabetical order.} \\
	Sorbonne Universit\'{e}, CNRS, LIP6 \\
	Paris, France \\
	\texttt{xinye.chen@lip6.fr} \\
	\And
	\href{https://orcid.org/0000-0003-1494-4478}{\includegraphics[scale=0.06]{figures/orcid.pdf}\hspace{1mm}Stefan G\"uttel} \\
	The University of Manchester \\
	Manchester, United Kingdom\\
	\texttt{stefan.guettel@manchester.ac.uk} \\
	\And
	\href{https://orcid.org/0009-0002-4319-2324}{\includegraphics[scale=0.06]{figures/orcid.pdf}\hspace{1mm}Mohammad Mozaffari} \\
	University of Toronto \\
	Toronto, Canada\\
	\texttt{mmozaffari@cs.toronto.edu} \\
}

\date{}

\renewcommand{\shorttitle}{Benchmarking Neural Memorization on Complexity-Controlled Symbolic Sequences}

\hypersetup{
pdftitle={Benchmarking Neural Memorization},
pdfsubject={cs.LG},
pdfauthor={Xinye Chen, Stefan G\"uttel, Mohammad Mozaffari},
pdfkeywords={Benchmarking, Transformer, RNNs, Symbolic sequences, Memorizing ability},
}

\begin{document}
\maketitle

\begin{abstract}
We introduce ROTE (RollOut Testing of Exact memorization), a benchmarking protocol for evaluating symbolic memorization of neural architectures.
We study memorization and the extension of symbolic rules in neural sequence models by using sequences whose complexity is regulated by Lempel--Ziv--Welch (LZW) compression.  Under ROTE, each architecture is trained as the same finite-context conditional predictor and is evaluated using teacher-forced one-step prediction as well as closed-loop rollout on the withheld symbols. Following a shared prediction-and-rollout evaluation routine, the benchmark evaluates gated recurrent, minimal recurrent, attention-based, and hybrid recurrent-attention models with their native computational characteristics preserved. Beyond standard predictive metrics, the benchmark reports normalized string distances, training time, memory usage, and parameter count across an LZW-complexity sweep.  The study establishes a connection between the complexity of algorithmic sequences and the memorization capacity of neural architectures, revealing the trade-offs involving memorization quality, rollout stability, and computational expense. Our software and reproducible experimental code can be obtained from \url{https://github.com/nla-group/rote}.
\end{abstract}


\section{Introduction}
RNNs and Transformers are important model architectures for sequence modeling, with applications across language, vision, audio, biology, weather, machine translation, and speech recognition \citep{kitaev-klein-2018-constituency,howard-ruder-2018-universal,Devlin2019,10.1145/3308558.3313707,doi:10.1073/pnas.2016239118,osti_2564318,zhou2021informer,valanarasu2022transweather,Vaswani2017,JMLR:v25:23-0870,dong2018speech,gulati2020conformer,prabhavalkar2024survey}.  Transformers have shown strong empirical performance through content-dependent attention, but standard self-attention has quadratic complexity in sequence length.  This cost has motivated efficient Transformer variants, including sparse, low-rank, reversible, and linear-attention mechanisms \citep{liu2018generating,kitaev2020reformer,wang2020linformer,DBLP:journals/corr/abs-2002-11296,katharopoulos2020transformers}.  The resulting design space makes it important to evaluate not only prediction accuracy, but also rule discovery, memorization, rollout stability, and computational cost under controlled sequence complexity.

Although Transformer models are considered to be a fundamentally distinct approach compared to recurrent neural networks, \cite{katharopoulos2020transformers} demonstrates that any Transformer layer employing causal masking can be treated as an RNN model. Transformers have been demonstrated to outperform  LSTM \citep{Hochreiter1997} and GRU \citep{Cho2014} architectures on many real-world tasks. It is natural to ask whether or not the RNN cells can be discarded entirely with the development of self-attention.  Finding rules consists of identifying the underlying patterns or relationships in the data, such as the syntactic structures in language, the temporal dependencies in time-series data, or the logical constraints in symbolic reasoning tasks.   The ability of a model to identify and generalize rules from sequential data is a  fundamental learning capability, particularly in tasks requiring reasoning, pattern recognition, or structured prediction. Measuring this ability is therefore important for guiding architecture design.

Existing benchmarking methods (e.g., \citep{tay2021long}) lack control over sequence complexity or structure and place minimal focus on computational efficiency or sequence similarity metrics. They often focus on natural language datasets (e.g., WikiText, Penn Treebank) and simple synthetic tasks (e.g., copy tasks, arithmetic sequences) to evaluate perplexity and generation quality for RNNs and Transformers by simulating various tasks like machine translation, speech recognition, and sentiment analysis. However, such benchmarks often have lower complexity or rely on statistical patterns that favor Transformer's attention mechanisms over RNN's sequential modeling, which often results in biased evaluation.
Furthermore, according to the scaling law \citep{kaplan2020scaling, hoffmann2022training}, optimal performance within a fixed computational budget is achieved by balancing model size with dataset scale, rather than solely increasing the number of parameters. However, these scaling laws overlook inference budgets, critical for large-scale deployment \citep{touvron2023llama}, and task complexity, such as sequence compressibility.

We propose a symbolic-sequence benchmark called \textbf{ROTE (\underline{R}oll\underline{O}ut \underline{T}esting of \underline{E}xact memorization)}, in which the difficulty of the seed pattern is controlled by Lempel--Ziv--Welch (LZW) compression before the sequence is periodically extended.  Each model receives the same length-$w$ context and is trained to predict the next symbol, then tested both on held-out one-step windows and on a closed-loop rollout generated from its own predictions.  This setup distinguishes two sources of difficulty that are often conflated: how compressible the seed is and whether the next symbol is identifiable from the chosen context length.  In the experiments, we report DL and JW rollout distances together with test loss, accuracy, parameter count, training time, time per epoch, and peak memory, so that failures can be read as errors in memorization, rollout stability, or compute use rather than as a single aggregate score.

The remainder of the paper is organized as follows: Section~\ref{sec:related} discusses the related work on RNNs and Transformer models, examining their learning capabilities and benchmarking performance, and highlighting the respective strengths and weaknesses of these architectures; Section~\ref{sec:benchmarking} illustrates our benchmark methods and the neural architectures we aim to simulate; Section~\ref{sec:simulation} presents empirical results and defines the outcomes; Section~\ref{sec:limitations} discusses limitations; Section~\ref{sec:conclusion} concludes and discusses future work.

\section{Related work}\label{sec:related}

The focus in natural language processing has shifted from using RNNs to using Transformers. RNNs, which include types such as LSTM and GRU, handle sequential data by keeping a hidden state, but they have difficulty with long-range dependencies because of their memory limitations and limited scalability \citep{Vaswani2017}. Transformers \citep{Vaswani2017} make use of multi-head attention mechanisms in order to process sequences in parallel and have achieved state-of-the-art results in tasks such as machine translation and text classification.

RNNs handle inputs in a sequential fashion by updating their hidden state in order to capture temporal dependencies, a feature that makes them suitable for tasks with a recurrent structure. Nevertheless, due to their sequential nature and the problem of vanishing gradients, they are limited in their ability to effectively model long sequences. Although LSTMs and GRUs mitigate this through gating mechanisms, they remain less scalable.  \cite{9679033} investigates how well RNNs can learn and represent languages generated by context-free grammars (CFGs), such as balanced parentheses and arithmetic expressions. It aims to bridge neural network sequence modeling with formal language theory, and it demonstrates that in principle RNNs can capture the hierarchical structures typical of CFGs. \cite{ROberto19} examines the architecture of RNNs by looking at the complexity of the string sequences that they are able to memorize; it finds that the learning rate and the number of units per layer are among the most important hyper-parameters to tune, and that GRUs perform better than LSTM networks on sequences of low complexity whereas LSTMs perform better on those of high complexity.

The Transformer model, as proposed in \cite{Vaswani2017}, makes use of self-attention in order to create representations that are dependent on the content within a given context window. Because standard self-attention scales quadratically with sequence length, many efficient attention mechanisms \citep{kitaev2020reformer,wang2020linformer,katharopoulos2020transformers,choromanski2021performer,tay2021long} have been proposed to reduce the cost of long-context modeling. However,  as shown in the Long Range Arena (LRA) benchmark \citep{tay2021long}, efficient attention models involve nontrivial trade-offs between quality and memory, rather than featuring a single design that is universally superior. More recent models have further blurred the line between attention and recurrence: linear attention reinterprets causal attention as a recurrent computation \citep{katharopoulos2020transformers}, Performer approximates softmax attention using random features \citep{choromanski2021performer}, RWKV combines recurrent time mixing with channel mixing in the style of Transformers \citep{peng2023rwkv}, and minGRU/minLSTM examine whether a simplified form of recurrence can achieve a large part of the effectiveness of gated RNNs \citep{feng2024rnnsneeded}.

\section{Models and benchmarking}\label{sec:benchmarking}

We formulate symbolic-sequence memorization as a controlled next-symbol prediction problem followed by a free-running rollout test.  The benchmark deliberately separates the source of sequence complexity, determined by a finite seed string, from the supervised learning problem, determined by finite-context conditional prediction.  This separation is important when comparing recurrent models with attention-based models: all architectures are evaluated through the same conditional distribution over the next symbol, while their internal state update, attention pattern, and pooling mechanism may differ.

\subsection{Controlled symbolic sequences}

Let
\begin{equation*}
\Sigma=\{\mathrm{A},\ldots,\mathrm{Z},\mathrm{a},\ldots,\mathrm{z}\},\qquad |\Sigma|=52,
\end{equation*}
and let $\Sigma_n=\{\sigma_1,\ldots,\sigma_n\}\subseteq\Sigma$ denote the active alphabet, $1\le n\le 52$.  A symbolic seed is a finite word $u=u_1\cdots u_m\in\Sigma_n^*$ with support
\begin{equation*}
\supp(u)=\{\sigma\in\Sigma_n:\exists i\in\{1,\ldots,m\},\ u_i=\sigma\}.
\end{equation*}
The generator is constrained to return seeds satisfying $|\supp(u)|=n$ and a target compression level $c$.  Hence alphabet size and algorithmic compressibility are controlled independently: $n$ fixes the output classification dimension, whereas $c$ controls the description length of the symbolic pattern to be memorized.

\subsubsection{LZW parsing, periodic reduction, and seed generation}

We use LZW compression as a deterministic finite-memory parsing procedure.  For a word $v\in\Sigma^*$, initialize a dictionary
\begin{equation*}
D_0:\Sigma\rightarrow\{0,\ldots,51\},
\end{equation*}
set the working phrase $p\leftarrow\epsilon$, and scan $v$ from left to right.  Whenever $pa$ is already in the dictionary, the parser extends the phrase, $p\leftarrow pa$.  Otherwise it emits the integer code $D(p)$, inserts the new phrase $pa$ with the next unused integer, and restarts from $p\leftarrow a$.  The final non-empty phrase is emitted at termination.  This defines a computable transduction
\begin{equation*}
\mathcal{C}_{\mathrm{LZW}}:\Sigma^*\rightarrow\mathbb{N}^*,
\qquad
v\mapsto (z_1,\ldots,z_\ell),
\end{equation*}
and we write
\begin{equation*}
\mathrm{LZW}(v)=|\mathcal{C}_{\mathrm{LZW}}(v)|
\end{equation*}
for the number of emitted codes.  The score is not used as an absolute Kolmogorov complexity; it is used as a reproducible, implementation-level proxy for symbolic description length.

Before compression, we quotient out exact repetition.  For $u\in\Sigma^m$, define its shortest exact period by
\begin{equation*}
\pi(u)=\min\left\{d\in\{1,\ldots,m\}: d\mid m\ \text{and}\ u_i=u_{i+d}\ \text{for all}\ 1\le i\le m-d\right\},
\end{equation*}
with $\pi(u)=m$ when no shorter divisor exists.  The reduction operator is
\begin{equation*}
\reduce(u)=u_{1:\pi(u)},
\end{equation*}
equivalently, $\reduce(u)=r$ if $u=r^q$ for some $q\ge2$ and $r$ has minimal length; otherwise $\reduce(u)=u$.  The benchmark complexity is therefore
\begin{equation}
\kappa(u)=\mathrm{LZW}(\reduce(u)).
\label{eq:lzw_complexity}
\end{equation}
This convention prevents a seed from appearing complex merely because it repeats the same block many times.  Computationally, $\reduce$ is implemented by enumerating candidate divisor lengths $d\le m/2$ and checking whether tiling the prefix $u_{1:d}$ reconstructs $u$ exactly.

The seed generator can be viewed as a randomized hitting-time procedure.  Given $(n,c)$, it first rejects infeasible settings with $n>c$ and caps $n$ at $52$.  It then draws a random permutation $S_0=\mathrm{Shuffle}(\Sigma_n)$ so that every symbol appears at least once.  For $t\ge0$, sample $A_t\sim\mathrm{Unif}(\Sigma_n)$ independently and set
\begin{equation*}
S_{t+1}=S_tA_t.
\end{equation*}
The stopping time is
\begin{equation*}
\tau_c=\inf\{t\ge0:\kappa(S_t)\ge c\},
\end{equation*}
and the returned seed is $u=S_{\tau_c}$ with recorded complexity $\kappa(u)$.  Since the loop stops at the first threshold crossing, $\kappa(u)$ is close to $c$ but may exceed it by one or more emitted LZW codes, depending on the phrase inserted by the final sampled symbol.  In the reported experiments, $c$ denotes the retained measured complexity: after generation, the scripts keep only seeds satisfying $\kappa(u)=c$ for the requested target and discard threshold-crossing seeds with $\kappa(u)>c$.  This convention makes the table and figure indices exact while preserving the randomized construction of candidate seeds.

Algorithm~\ref{alg:lzw} specifies the compression transducer, Algorithm~\ref{alg:reduce} specifies exact-period reduction, and Algorithm~\ref{algo:lwz_generation} gives the executable generator.

\begin{minipage}[t]{0.45\textwidth}
\begin{algorithm}[H]
\caption{LZW Compression Transducer}
\label{alg:lzw}
\begin{algorithmic}[1]
\everymath{\color{AlgVarColor}}
\Require Word $v=v_1\cdots v_m\in\Sigma^*$
\Ensure Code sequence $C(v)$
\State $D\gets\{\sigma\mapsto i:\sigma\in\Sigma,\ i=0,\ldots,51\}$
\State $p\gets\epsilon$, $C\gets[\ ]$
\For{$i=1$ to $m$}
    \State $a\gets v_i$
    \If{$pa\in\mathrm{dom}(D)$}
        \State $p\gets pa$
    \Else
        \State append $D(p)$ to $C$
        \State $D(pa)\gets |D|$
        \State $p\gets a$
    \EndIf
\EndFor
\If{$p\ne\epsilon$}
    \State append $D(p)$ to $C$
\EndIf
\State \Return $C$
\end{algorithmic}
\end{algorithm}

\begin{algorithm}[H]
\caption{Exact-Period Reduction}
\label{alg:reduce}
\begin{algorithmic}[1]
\everymath{\color{AlgVarColor}}
\Require Word $u=u_1\cdots u_m\in\Sigma^*$
\Ensure Reduced word $\reduce(u)$
\For{$d=1$ to $\lfloor m/2\rfloor$}
    \If{$d\mid m$ and $u_i=u_{i+d}$ for all $1\le i\le m-d$}
        \State \Return $u_{1:d}$
    \EndIf
\EndFor
\State \Return $u$
\end{algorithmic}
\end{algorithm}
\end{minipage}
\hfill
\begin{minipage}[t]{0.45\textwidth}
\begin{algorithm}[H]
\caption{LZW-Controlled Seed Generator}
\label{algo:lwz_generation}
\begin{algorithmic}[1]
\everymath{\color{AlgVarColor}}
\Require Alphabet size $n$, target complexity $c$, random seed $r$
\Ensure Seed $u$, measured complexity $\kappa(u)$
\If{$n>52$}
    \State $n\gets52$
\EndIf
\If{$n=1$}
    \State \Return $(\mathrm{A},1)$
\EndIf
\If{$n>c$}
    \State \Return $(\mathrm{invalid},0)$
\EndIf
\State initialize pseudo-random generator with $r$
\State $\Sigma_n\gets\{\sigma_1,\ldots,\sigma_n\}$
\State $u\gets\mathrm{Shuffle}(\Sigma_n)$
\State $\kappa\gets\mathrm{LZW}(\reduce(u))$
\While{$\kappa<c$}
    \State draw $a\sim\mathrm{Unif}(\Sigma_n)$
    \State $u\gets ua$
    \State $\kappa\gets\mathrm{LZW}(\reduce(u))$
\EndWhile
\State \Return $(u,\kappa)$
\end{algorithmic}
\end{algorithm}
\end{minipage}

\subsection{Finite-context identifiability}

The LZW complexity controls the description length of a seed, but it does not by itself determine whether the next-symbol task is identifiable from a context of length $w$.  To make this distinction explicit, let $u\in\Sigma_n^m$ denote the reduced seed and interpret indices cyclically, i.e., $u_i=u_{1+((i-1)\bmod m)}$ for all $i\in\mathbb{Z}$.  For a context length $r\ge1$, define the set of cyclic contexts
\begin{equation*}
\mathcal{Z}_r(u)=\left\{u_{i-r+1:i}: i=1,\ldots,m\right\}\subseteq\Sigma_n^r
\end{equation*}
and the set of admissible next symbols after a context $z\in\mathcal{Z}_r(u)$ by
\begin{equation*}
\mathrm{Next}_r(z;u)=
\left\{u_{i+1}:u_{i-r+1:i}=z,\ i=1,\ldots,m\right\}.
\end{equation*}
The finite-context ambiguity rate is
\begin{equation}
\alpha_r(u)=\frac{1}{m}\sum_{i=1}^{m}
\mathbf{1}\left\{
\left|\mathrm{Next}_r(u_{i-r+1:i};u)\right|>1
\right\}.
\label{eq:ambiguity_rate}
\end{equation}
Thus $\alpha_r(u)$ measures the fraction of phases whose length-$r$ suffix does not uniquely determine the next symbol.  The minimal predictive order is
\begin{equation}
r^\star(u)=\min\{r\ge1:\alpha_r(u)=0\}.
\label{eq:predictive_order}
\end{equation}
For a primitive reduced seed, $r^\star(u)\le m$, but it can be much smaller than $m$ when short contexts already determine the symbolic dynamics.

\begin{proposition}[Finite-context identifiability]
For a periodic symbolic source generated by $u$, there exists a deterministic map $g_r:\Sigma_n^r\rightarrow\Sigma_n$ such that
\begin{equation*}
g_r(u_{i-r+1:i})=u_{i+1},\qquad i=1,\ldots,m,
\end{equation*}
if and only if $\alpha_r(u)=0$.  Consequently, $r^\star(u)$ is the smallest context length for which perfect one-step prediction is identifiable from finite contexts.
\end{proposition}

\noindent\emph{Proof sketch.}  If $\alpha_r(u)=0$, every observed context $z$ has a singleton next-symbol set, so $g_r(z)$ can be defined as that unique element and extended arbitrarily outside $\mathcal{Z}_r(u)$.  Conversely, if some context has two distinct admissible next symbols, no deterministic function of that context alone can match both transitions.

For probabilistic models, the same obstruction appears as conditional entropy.  Let $\widehat{p}_r(z,a)$ be the empirical distribution over cyclic context--next-symbol pairs and let $\widehat{p}_r(a\mid z)$ be the induced conditional law.  The Bayes-optimal empirical cross-entropy at context length $r$ is
\begin{equation}
H_r(u)=
-\sum_{z\in\mathcal{Z}_r(u)}\widehat{p}_r(z)
\sum_{a\in\Sigma_n}\widehat{p}_r(a\mid z)
\log \widehat{p}_r(a\mid z),
\label{eq:context_entropy}
\end{equation}
which vanishes exactly when every context has a unique next symbol.  Hence $\kappa(u)$ measures compressibility, whereas $r^\star(u)$ and $\alpha_w(u)$ measure whether the supervised finite-context task is intrinsically learnable at the chosen window length.

\subsection{Finite-context prediction and rollout}

For each seed string $u$, define the length-$N$ periodic extension operator $\mathrm{Per}_N$ by
\begin{equation*}
x=\mathrm{Per}_N(u),\qquad
x_i=u_{1+((i-1)\bmod |u|)},\quad i=1,\ldots,N.
\end{equation*}
Let $w$ denote the context-window length and $k$ the forecast horizon.  The final target block
\begin{equation*}
y^\star=x_{N-k+1:N}\in\Sigma_n^k
\end{equation*}
is removed from supervised training and used only for sequence-level forecasting.  The observable prefix is $x^{\mathrm{obs}}=x_{1:N-k}$.

Let $\phi:\Sigma_n\rightarrow\{e_1,\ldots,e_n\}\subset\mathbb{R}^n$ be the one-hot encoder and extend it tokenwise to windows by
\begin{equation*}
\Phi_w(x_{i-w+1:i})=
\begin{bmatrix}
\phi(x_{i-w+1})^\top\\
\vdots\\
\phi(x_i)^\top
\end{bmatrix}\in\{0,1\}^{w\times n}.
\end{equation*}
The supervised examples are the ordered pairs
\begin{equation}
\mathcal{D}(x;w,k)=
\left\{
\left(X_i,y_i\right):
X_i=\Phi_w(x_{i-w+1:i}),\ y_i=x_{i+1},\ w\le i\le N-k-1
\right\}.
\label{eq:window_dataset}
\end{equation}
The split into $\mathcal{D}_{\mathrm{tr}}$, $\mathcal{D}_{\mathrm{val}}$, and $\mathcal{D}_{\mathrm{te}}$ is chronological, so later windows are never used to predict earlier windows.  This is the finite-context analogue of a one-step-ahead time-series protocol.

Every architecture is wrapped as the same scorer
\begin{equation*}
f_\theta:\{0,1\}^{w\times n}\rightarrow\mathbb{R}^{n},
\end{equation*}
where the $a$th logit scores the next symbol $\sigma_a$.  The conditional distribution is
\begin{equation}
P_\theta(\sigma_a\mid X)=
\frac{\exp(f_\theta(X)_a)}
{\sum_{b=1}^{n}\exp(f_\theta(X)_b)},
\qquad a=1,\ldots,n.
\label{eq:conditional_model}
\end{equation}
The architectures differ only in how $f_\theta$ is computed: gated recurrent models update hidden states sequentially; minimal recurrent models simplify the gating map while preserving a recurrent state; softmax and linear-attention Transformers compute contextual representations over the fixed window; Performer uses random-feature attention; and RWKV-style models combine recurrent time mixing with channel mixing.  No model is given access to $y^\star$ during fitting.

The fitted parameters minimize the empirical negative conditional log-likelihood
\begin{equation}
\hat{\theta}\in\arg\min_\theta
\left\{
-\frac{1}{|\mathcal{D}_{\mathrm{tr}}|}
\sum_{(X_i,y_i)\in\mathcal{D}_{\mathrm{tr}}}
\log P_\theta(y_i\mid X_i)
\right\},
\label{eq:nll_objective}
\end{equation}
optimized by mini-batch first-order methods.  Validation loss is used only for hyperparameter selection and early stopping, while $\mathcal{D}_{\mathrm{te}}$ provides teacher-forced one-step loss and accuracy.

After training, memorization is tested by closed-loop decoding.  Initialize
\begin{equation*}
c_0=x_{N-k-w+1:N-k}\in\Sigma_n^w.
\end{equation*}
For $j=1,\ldots,k$, form $X(c_{j-1})=\Phi_w(c_{j-1})$ and generate greedily
\begin{equation}
\hat{x}_{N-k+j}
=\arg\max_{\sigma_a\in\Sigma_n} P_{\hat{\theta}}(\sigma_a\mid X(c_{j-1})),
\qquad
c_j=(c_{j-1,2},\ldots,c_{j-1,w},\hat{x}_{N-k+j}).
\label{eq:rollout}
\end{equation}
Ties in the maximization are resolved by the fixed alphabet order $\sigma_1,\ldots,\sigma_n$.  The rollout forecast is $\hat{y}=\hat{x}_{N-k+1:N}$.  Teacher-forced metrics evaluate the local conditional model in Equation~\eqref{eq:conditional_model}; rollout metrics evaluate the stability of the induced dynamical system when predictions are recursively fed back as inputs.

\subsubsection{Teacher-forced error versus rollout stability}
Teacher-forced evaluation samples contexts from the true trajectory, while closed-loop rollout samples contexts from the model-induced trajectory.  When an early prediction is wrong, the next context may leave the empirical context distribution even if the one-step test error is small.  Let $E_j$ denote the event that the first $j$ rollout symbols are correct, with $E_0$ the sure event, and define the conditional rollout error
\begin{equation*}
\epsilon_j=\Pr(E_j^c\mid E_{j-1}),\qquad j=1,\ldots,k.
\end{equation*}
Then the exact-reconstruction probability satisfies
\begin{equation}
\Pr(\hat{y}=y^\star)=\Pr(E_k)=\prod_{j=1}^{k}(1-\epsilon_j),
\qquad
1-\Pr(\hat{y}=y^\star)\le \sum_{j=1}^{k}\epsilon_j.
\label{eq:rollout_amplification}
\end{equation}
This elementary chain-rule identity shows why a model can have high teacher-forced accuracy but poor long-horizon reconstruction.  Rollout metrics therefore test not only the learned local transition rule, but also the stability of that rule under its own generated contexts.

\subsection{A complexity-aware scaling view}

Classical neural scaling laws model error primarily as a function of trainable parameter count $P_\theta$, training-token budget $N_{\mathrm{tok}}$, and compute budget $B_{\mathrm{comp}}$ \citep{kaplan2020scaling,hoffmann2022training}.  In the present benchmark, each data point also has an algorithmic and finite-context profile
\begin{equation*}
\zeta(u;w,k)=\left(n,\kappa(u),r^\star(u),\alpha_w(u),w,k\right).
\end{equation*}
For a model family $\mathcal{M}$, the relevant response surface is therefore better viewed as
\begin{equation}
\mathcal{E}_{\mathcal{M}}
=
F_{\mathcal{M}}\left(P_\theta,N_{\mathrm{tok}},B_{\mathrm{comp}}\mid \zeta(u;w,k)\right),
\label{eq:complexity_aware_scaling}
\end{equation}
rather than as a function of $(P_\theta,N_{\mathrm{tok}},B_{\mathrm{comp}})$ alone.  Equation~\eqref{eq:complexity_aware_scaling} is not proposed as a fitted law in this paper; it is the organizing principle behind the benchmark.  It explains why model-size comparisons are ambiguous unless the sequence complexity, context length, and rollout horizon are reported together, and it motivates future sweeps in which $P_\theta$, $N_{\mathrm{tok}}$, $w$, $k$, $\kappa(u)$, and $r^\star(u)$ are varied jointly.

\section{Experiments}\label{sec:simulation}

\subsection{Experimental protocol}
We run a focused architecture comparison on the LZW-controlled symbolic benchmark.  The evaluated models are LSTM, GRU, minGRU, minLSTM, a softmax-attention Transformer, a linear-attention Transformer, Performer, and RWKV.  This set was chosen to compare three families of inductive bias under a common prediction protocol: explicit recurrent state, finite-window attention, and hybrid or linearized alternatives to quadratic attention.

For each alphabet size $n\in\{2,4,6,8\}$ and target LZW complexity $c\in\{10,30,50,70,90\}$, we generate a seed string, extend it periodically to length $N=3500$, construct windows of length $w=100$, and reserve the final $k=100$ symbols for closed-loop rollout.  We use two stochastic seeds and two independent training runs for each architecture--complexity--alphabet configuration, yielding $8\times4\times5\times2\times2=640$ trained models.  All models have two layers.  LSTM, GRU, minGRU, and minLSTM use recurrent width $128$; Transformer, LinearAttention, Performer, and RWKV use model dimension $256$.  We optimize with AdamW, learning rate $3\times10^{-4}$, weight decay $10^{-2}$, batch size $128$, a maximum of 200 epochs, early-stopping patience 10, and stopping loss threshold $0.05$.

All experiments were run with CUDA~13.0 and cuDNN~9.22.0 on a single NVIDIA A100 GPU.  During these runs, Performer was used with causal attention to match the autoregressive next-symbol prediction protocol; its implementation fell back to the non-CUDA autoregressive kernel, so the reported memory and runtime for Performer should be interpreted as conservative implementation-level costs rather than as hardware-optimal Performer costs.

Following scaling-law practice \citep{kaplan2020scaling,hoffmann2022training}, we do not treat parameter count alone as a sufficient explanation of quality.  Instead, each architecture is represented by the tuple
\begin{equation*}
\left(P_\theta,\ T_{\mathrm{train}},\ T_{\mathrm{epoch}},\ M_{\mathrm{peak}},
\mathcal{L}_{\mathrm{te}},\ A_{\mathrm{te}},\ \mathrm{DL},\ \mathrm{JW}\right),
\end{equation*}
where $P_\theta$ is the number of trainable parameters, $T_{\mathrm{train}}$ is total training time, $T_{\mathrm{epoch}}$ is time per epoch, $M_{\mathrm{peak}}$ is peak memory, $\mathcal{L}_{\mathrm{te}}$ and $A_{\mathrm{te}}$ are teacher-forced test loss and accuracy, and DL and JW are normalized rollout distances.  This fixed-budget design is not a compute-optimal scaling sweep; it is a controlled stress test for architecture-induced differences as symbolic description length varies.

To check whether the main ranking is driven by unequal parameter counts, we additionally run a high-complexity matched-size experiment.  In this auxiliary track, we fix $c=90$, $n\in\{4,8\}$, $N=3500$, $w=k=100$, two generated seeds, and two independent training runs.  For each alphabet size and each architecture in the set \{LSTM, GRU, Transformer, LinearAttention, Performer, RWKV\}, we select the closest configuration to a target size $P_0=0.2\,\mathrm{M}$ trainable parameters from the same width grids used by the experiment code.  This yields 48 additional fits under the same optimizer, stopping rule, CUDA version, and cuDNN version as the main study.  The matched-size track is a robustness check, not a full scaling-law sweep: it tests one hard symbolic regime while holding the training-token budget and evaluation protocol fixed.

\subsection{Accuracy under increasing symbolic complexity}
Figure~\ref{fig:complexity_sweep} shows that LZW complexity induces a clear degradation in both local prediction and closed-loop rollout.  Averaged across models, the normalized DL distance increases from $0.003$ at $c=10$ to $0.099$ at $c=70$, while test accuracy decreases from $0.999$ to $0.922$.  The non-monotonic improvement from $c=70$ to $c=90$ indicates that LZW description length is a useful but not total ordering of difficulty: alphabet size, seed periodicity, finite-context ambiguity $\alpha_w(u)$, and predictive order $r^\star(u)$ can vary across seeds with the same or larger compression threshold.

Table~\ref{tab:seed_diagnostics} reports these seed-level quantities for the generated strings.  The benchmark window $w=100$ gives $\alpha_{100}(u)=0$ for all seeds, so the supervised one-step task is identifiable from the chosen context length.  However, the seed length and the predictive order still increase with $c$, and the spread across alphabet sizes remains substantial.  In particular, the transition from $c=70$ to $c=90$ increases median seed length and median $r^\star(u)$ only mildly, which helps explain why the empirical curves need not be monotone in $c$ alone.

\input{tables/seed_diagnostics}

\begin{figure}[t]
    \centering
    \includegraphics[width=1\textwidth]{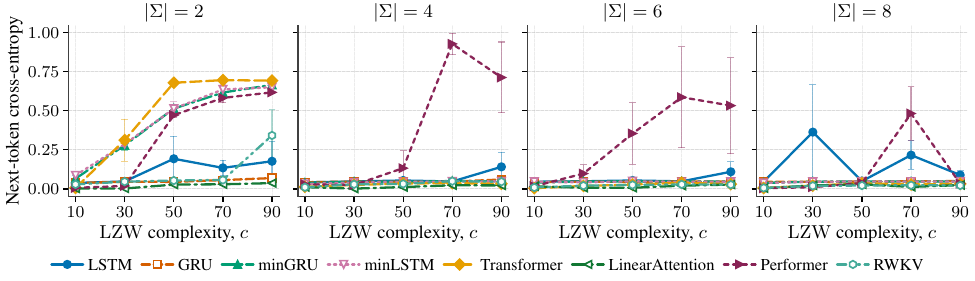}
    \includegraphics[width=1\textwidth]{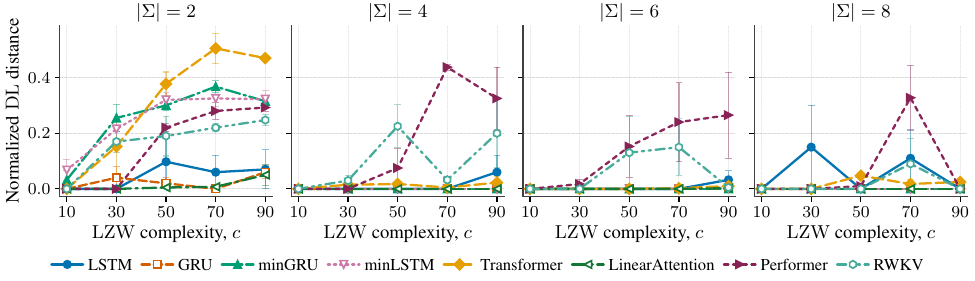}
    \caption{Teacher-forced loss and closed-loop rollout error as functions of LZW complexity.  Both panels use the same model identities and summarize the 640 training runs.  Lower values are better for both metrics.}
    \label{fig:complexity_sweep}
\end{figure}

The architecture ranking is not determined by model size.  GRU and LinearAttention are the strongest overall: their exact 100-step rollout rates are $0.925$ and $0.912$, respectively, and at the hardest reported complexity $c=90$ they obtain mean DL distances $0.015$ and $0.013$.  LSTM remains competitive with exact rollout rate $0.887$, but its higher test loss at $c=90$ suggests less stable optimization under the same budget.  The vanilla Transformer, despite having the largest parameter count in the experiment, has an overall exact rollout rate of $0.613$ and mean DL distance $0.132$ at $c=90$.  Thus, on this memorization benchmark, attention capacity without a matching inductive bias or compute allocation does not guarantee stable symbolic extrapolation.

\begin{figure}[ht!]
    \centering
    \includegraphics[width=0.62\textwidth]{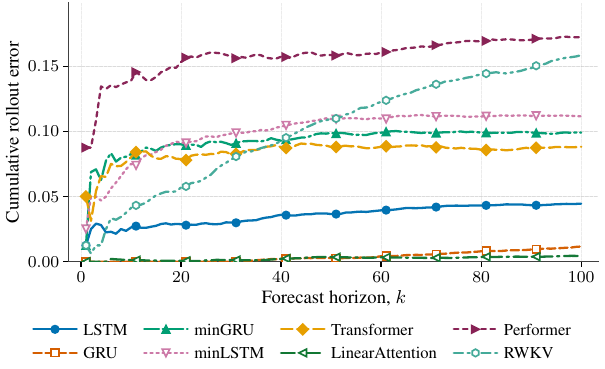}
    \caption{Closed-loop degradation over the rollout horizon.  Errors introduced early in the rollout become part of the subsequent context, so this plot measures stability of the learned symbolic dynamics rather than one-step classification alone.  Markers are shown every ten rollout steps to improve readability; the curves are computed over all horizons.}
    \label{fig:rollout_horizon}
\end{figure}

\subsection{Compute--accuracy trade-offs}
Table~\ref{tab:model_summary} and Figure~\ref{fig:efficiency} summarize the performance--cost frontier.  minGRU is the smallest and fastest model in wall-clock time, with median $0.067$M parameters and $6.30$s training time, but this speed comes with lower rollout fidelity than GRU or LinearAttention.  LinearAttention reaches near-GRU rollout quality with substantially fewer epochs and lower wall-clock time, although its memory footprint is larger in the present implementation.  RWKV converges quickly and uses less memory than attention baselines, but its rollout error remains higher than the best gated recurrent and linear-attention models.

\input{tables/model_summary}

\begin{figure}[ht!]
    \centering
    \includegraphics[width=0.495\textwidth]{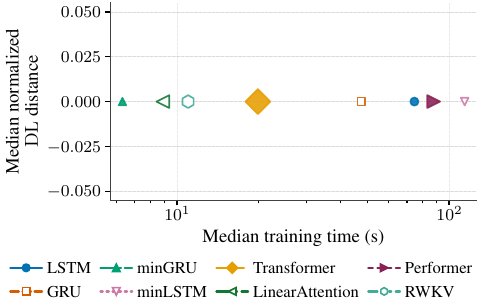}
    \includegraphics[width=0.495\textwidth]{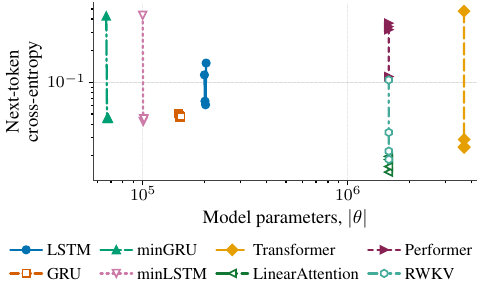}
    \includegraphics[width=0.495\textwidth]{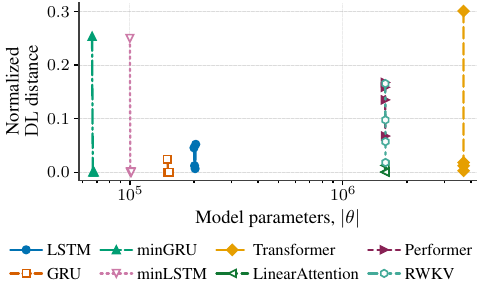}
    \caption{Efficiency views of the same benchmark results.  The Pareto plot compares rollout quality against measured compute, while the parameter plots show that increasing $P_\theta$ alone does not explain the observed ordering.}
    \label{fig:efficiency}
\end{figure}

\subsection{Matched-size high-complexity check}\label{sec:matched_size}
The main comparison intentionally preserves each architecture's native computational profile, which means that the parameter counts in Table~\ref{tab:model_summary} are not equal.  Table~\ref{tab:matched_size} reports the auxiliary matched-size check at $c=90$.  After reducing all models to roughly the same scale, the qualitative ordering does not reverse.  LSTM, GRU, LinearAttention, and Performer exactly reproduce every 100-symbol rollout in this small high-complexity track.  The vanilla Transformer misses one of eight rollouts, while RWKV remains accurate in teacher-forced prediction but is less stable in closed-loop generation.  LinearAttention is the fastest exact model in this setting, whereas Performer is exact but expensive because the run uses the non-CUDA autoregressive fallback.  Appendix Figure~\ref{fig:matched_size_appendix} shows the same comparison against the fixed-budget high-complexity subset for both rollout distance and teacher-forced loss.

\input{tables/matched_size}

These results support three conclusions.  First, controlled compressibility is a meaningful axis for symbolic sequence benchmarks, but it should be interpreted together with $r^\star(u)$ and $\alpha_w(u)$ because finite-window identifiability is the immediate learning constraint.  Second, recurrence remains a strong inductive bias for finite-context symbolic memorization, with GRU forming the most reliable low-memory baseline.  Third, efficient attention can be competitive when the architecture preserves the autoregressive structure of the task, as shown by LinearAttention.  The matched-size check sharpens this point: reducing the parameter-count gap removes one possible confound, but it does not by itself explain the remaining gap.  Larger softmax-attention models still require a more careful scaling of tokens, parameters, and optimization budget before their parameter advantage can be interpreted fairly.

\section{Limitations}\label{sec:limitations}

The comparison controls the data, optimizer settings, and stopping rules, but not the realized number of optimizer updates: early stopping gives different training durations across runs. Neither track is therefore a compute-matched or architecture-specifically tuned comparison. The high-complexity check in Section~\ref{sec:matched_size} approximately matches parameter counts around $P_0=0.2\,\mathrm{M}$ for six model families; it does not fix the number of training tokens per parameter and does not include minGRU or
minLSTM. In particular, the Transformer and Performer results should not be interpreted as their best attainable performance.

The benchmark studies models with at most a few million parameters on periodically extended symbolic sequences. The withheld suffix lies within the same period, so the task probes reproduction of a known pattern rather than discovery of an unknown rule, and these results should not be extrapolated to natural language data or large-scale training. For each main configuration two generated seeds are used together with two training runs for each seed. The matched-size check includes eight evaluations per architecture for two different alphabet sizes, but the runs that share a seed are not independent samples of symbolic structure. Exact reconstruction yields a binary result; the reported standard deviations are descriptive  rather than confidence intervals, and the small differences between the leading models are still uncertain.

For all tested seeds, $r^\star(u)\leq18$ and $\alpha_{100}(u)=0$. This means the experiments do not test for intrinsic finite-context ambiguity. The case where $w<r^\star(u)$, so $\alpha_w(u)>0$ and perfect one-step prediction is not possible, is still unexplored. Also, elapsed time and peak PyTorch CUDA usage depend on the specific implementation and the A100 GPU slice used. Performer was run without the optional custom causal CUDA kernel, so its measured costs should not be compared to optimized versions. State-space models like Mamba are not included in this comparison.

\section{Conclusion}\label{sec:conclusion}
We presented a symbolic-sequence  benchmark for studying the memorization ability of neural models. The benchmark controls input complexity via LZW compression and evaluates neural models through both teacher-forced prediction and closed-loop rollout.  The empirical study shows that the best model is not the largest model under a fixed budget: GRU and LinearAttention dominate the main trade-off between accuracy and cost, while minGRU is highly efficient but less reliable for exact long-horizon reconstruction.  The vanilla Transformer and Performer baselines do not automatically benefit from their larger parameter counts in this finite-data, finite-context setting, highlighting the need to report model size, training tokens, memory, and runtime together.

The empirical result of our extensive experiments is methodological.  Scaling laws already tell us that model size and data budget need to be considered together; our findings further show that task complexity, predictive order, finite-context ambiguity, and rollout horizon must likewise be stated clearly when studying symbolic sequence learning.  A full matched-scaling analysis would additionally sweep sequence length, training-token multiplier, and parameter quantity, and would revisit efficient attention and hybrid recurrent models with optimized kernels.  Such studies would turn the present benchmark from a fixed-budget comparison with a targeted robustness check into a compute-optimal scaling analysis for algorithmic memorization.

\bibliographystyle{unsrtnat}
\bibliography{references}

\appendix
\section{Appendix}

\subsection{Fixed-budget and matched-size comparison}
Figure~\ref{fig:matched_size_appendix} compares the main fixed-budget experiment with the matched-size robustness check on the common high-complexity subset $c=90$ and $n\in\{4,8\}$.  The two panels report the two complementary measurements used throughout the benchmark: closed-loop rollout distance and teacher-forced test loss.

\begin{figure}[t]
    \centering
    \includegraphics[width=\textwidth]{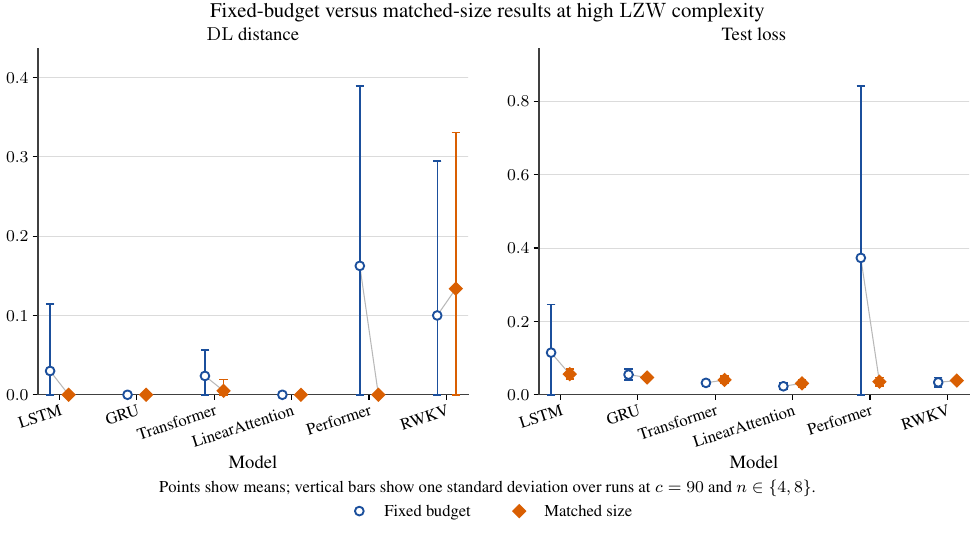}
    \caption{Fixed-budget and matched-size results on the common high-complexity subset $c=90$, $n\in\{4,8\}$.  Points show means and vertical bars show one standard deviation over two seeds and two runs for each alphabet size.  The left panel measures closed-loop symbolic stability through normalized DL distance; the right panel measures local teacher-forced prediction through test loss.}
    \label{fig:matched_size_appendix}
\end{figure}

\subsection{Distance Metrics}
Let $\hat{y},y^\star\in\Sigma^k$ denote the rollout forecast and the withheld target string.  Since $\Sigma^k$ is not naturally embedded in a Euclidean space, we evaluate sequence-level prediction using normalized string distances.  The reported DL metric is a normalized Damerau--Levenshtein distance, and the reported JW metric is a normalized Jaro--Winkler distance.  Smaller values are better for both metrics; exact reconstruction gives $\mathrm{DL}=\mathrm{JW}=0$.

The Damerau--Levenshtein edit distance $d_{\mathrm{edit}}(i,j)$ between prefixes $\hat{y}_{1:i}$ and $y^\star_{1:j}$ is defined recursively by
\begin{equation}
d_{\mathrm{edit}}(i,j) = \min \begin{cases}
0, & i=j=0, \\
d_{\mathrm{edit}}(i-1,j) + 1, & i>0, \\
d_{\mathrm{edit}}(i,j-1) + 1, & j>0, \\
d_{\mathrm{edit}}(i-1,j-1) + \mathbf{1}\{\hat{y}_i\ne y^\star_j\}, & i,j>0, \\
d_{\mathrm{edit}}(i-2,j-2) + 1, & i,j>1,\ \hat{y}_i=y^\star_{j-1},\
\hat{y}_{i-1}=y^\star_j .
\end{cases}
\end{equation}
For equal-length strings, the normalized score is
\begin{equation}
\mathrm{DL}(\hat{y},y^\star)=\frac{d_{\mathrm{edit}}(k,k)}{k}\in[0,1].
\end{equation}

For Jaro--Winkler, let $\mu$ be the number of matching symbols, where two symbols are considered matching if they are equal and their positions differ by at most
\begin{equation*}
\left\lfloor\frac{\max(|\hat{y}|,|y^\star|)}{2}\right\rfloor-1,
\end{equation*}
and let $t$ be half the number of transpositions among matched symbols.  The Jaro similarity is
\begin{equation}
\mathrm{Jaro}(\hat{y},y^\star)=
\begin{cases}
0, & \mu=0,\\
\frac{1}{3}\left(\frac{\mu}{|\hat{y}|}+\frac{\mu}{|y^\star|}
+\frac{\mu-t}{\mu}\right), & \mu>0.
\end{cases}
\end{equation}
Let $\ell\le 4$ be the common-prefix length and let $\rho=0.1$ be the standard prefix-scaling parameter.  The Jaro--Winkler similarity is
\begin{equation}
S_{\mathrm{JW}}(\hat{y},y^\star)=
\mathrm{Jaro}(\hat{y},y^\star)+\ell\rho\left(1-\mathrm{Jaro}(\hat{y},y^\star)\right).
\end{equation}
We report the corresponding distance
\begin{equation}
\mathrm{JW}(\hat{y},y^\star)=1-S_{\mathrm{JW}}(\hat{y},y^\star).
\end{equation}

\subsection{Model architectures}
All architectures instantiate the scorer $f_\theta$ in Equation~\eqref{eq:conditional_model} and map a one-hot context $X\in\{0,1\}^{w\times n}$ to next-symbol logits in $\mathbb{R}^n$.  The benchmark therefore controls the input--output contract while allowing each model family to use its native state representation.

\begin{itemize}
    \item \textbf{LSTM and GRU.}  The classical recurrent baselines process the window sequentially and predict from the final hidden state.  LSTM uses input, forget, and output gates with an explicit cell state \citep{Hochreiter1997}; GRU uses reset and update gates with a lighter hidden-state update \citep{Cho2014}.  These models have linear cost in the window length and expose a direct recurrent inductive bias for periodic symbolic dynamics.

    \item \textbf{minGRU and minLSTM.}  The minimal recurrent baselines follow the recent view that simplified gated recurrences can retain much of the sequence-modeling advantage of RNNs while improving parallelizability and implementation simplicity \citep{feng2024rnnsneeded}.  In our benchmark they use the same finite-context input and next-symbol objective as the classical recurrent models.

    \item \textbf{Transformer.}  The softmax-attention Transformer projects the context window to $d_m$ dimensions, adds positional information, applies multi-head self-attention and feed-forward blocks, and predicts from the final contextual representation \citep{Vaswani2017}.  Its attention cost is $O(w^2d_m)$ per layer, which is visible in the memory--accuracy analysis.

    \item \textbf{LinearAttention.}  LinearAttention replaces softmax attention by a kernelized attention map whose associativity permits $O(wd_m^2)$ or linear-in-length autoregressive computation, depending on implementation details \citep{katharopoulos2020transformers}.  It is included to test whether an attention model with recurrent-style computation can match gated recurrence on symbolic memorization.

    \item \textbf{Performer.}  Performer approximates softmax attention through positive random features, reducing the nominal attention cost from quadratic to linear in sequence length \citep{choromanski2021performer}.  The measured run used the available non-CUDA autoregressive fallback, so its empirical cost should be read as implementation-specific.

    \item \textbf{RWKV.}  RWKV combines Transformer-like channel mixing with an RNN-like time-mixing recurrence \citep{peng2023rwkv}.  It is included as a modern hybrid baseline that keeps constant-size recurrent state at inference while retaining several Transformer design elements.
\end{itemize}

\end{document}

%% file: tables/seed_diagnostics.tex
\begin{table}[t]
\centering
\caption{Seed diagnostics grouped by target LZW complexity.  Statistics summarize the eight generated seeds at each $c$ value, corresponding to four alphabet sizes and two stochastic seeds.  Length and $r^\star$ are reported as median with range; $\alpha_{100}^{\max}$ is the maximum ambiguity rate at the benchmark context length $w=100$.}
\label{tab:seed_diagnostics}
\begin{tabular}{rrrr}
\toprule
$c=\kappa(u)$ & $|u|$ median (range) & $r^\star(u)$ median (range) & $\alpha_{100}^{\max}$ \\
\midrule
10 & 11.0 (10--18) & 3.0 (2--9) & 0.000 \\
\shaderow 30 & 42.0 (33--82) & 5.0 (4--12) & 0.000 \\
50 & 83.5 (65--171) & 6.0 (4--14) & 0.000 \\
\shaderow 70 & 127.5 (97--263) & 6.5 (4--16) & 0.000 \\
90 & 174.0 (127--365) & 7.0 (4--18) & 0.000 \\
\bottomrule
\end{tabular}
\end{table}

%% file: tables/model_summary.tex
\begin{table}[t]
\centering
\caption{Model-level summary over all alphabet sizes and LZW complexities.  Exact denotes the fraction of runs for which the 100-symbol rollout exactly matches the withheld suffix.  Exact and DL@$c=90$ are reported as mean $\pm$ standard deviation over successful evaluations (runs that completed training); cost statistics are medians.}
\label{tab:model_summary}
\begin{tabular}{lrrrrrr}
\toprule
Model & Exact & DL@$c=90$ & Params (M) & Time (s) & Time/epoch (s) & Memory (MB) \\
\midrule
GRU & $0.925\pm0.265$ & $0.015\pm0.034$ & 0.152 & 47.59 & 1.11 & 187.30 \\
\shaderow LinearAttention & $0.912\pm0.284$ & $0.013\pm0.040$ & 1.581 & 8.88 & 1.27 & 2407.24 \\
LSTM & $0.887\pm0.318$ & $0.041\pm0.092$ & 0.202 & 74.56 & 1.14 & 213.13 \\
\shaderow minGRU & $0.787\pm0.412$ & $0.079\pm0.141$ & 0.067 & 6.30 & 0.11 & 196.04 \\
minLSTM & $0.762\pm0.428$ & $0.082\pm0.146$ & 0.101 & 114.09 & 2.78 & 130.48 \\
\shaderow RWKV & $0.613\pm0.490$ & $0.113\pm0.161$ & 1.585 & 10.99 & 1.68 & 656.22 \\
Transformer & $0.613\pm0.490$ & $0.132\pm0.203$ & 3.690 & 19.83 & 1.54 & 1313.01 \\
\shaderow Performer & $0.575\pm0.497$ & $0.221\pm0.217$ & 1.582 & 87.78 & 6.28 & 3782.05 \\
\bottomrule
\end{tabular}
\end{table}

%% file: tables/matched_size.tex
\begin{table}[t]
\centering
\caption{High-complexity matched-size robustness check at $c=90$.  Each row summarizes eight evaluations: two alphabet sizes $n\in\{4,8\}$, two generated seeds, and two training runs.  Widths are chosen to approximate a target size $P_0=0.2\,\mathrm{M}$ parameters.  Exact is the fraction of 100-step rollouts with zero DL distance.  DL, test loss, train time, and memory are reported as mean $\pm$ standard deviation.}
\label{tab:matched_size}
\small
\begin{tabular}{lrrrrrr}
\toprule
Model & Params (M) & Exact & DL & Test loss & Time (s) & Memory (MB) \\
\midrule
GRU & $0.236\pm0.001$ & 1.000 & $0.000\pm0.000$ & $0.047\pm0.003$ & $58.10\pm9.93$ & $230.59\pm3.12$ \\
\shaderow LSTM & $0.203\pm0.001$ & 1.000 & $0.000\pm0.000$ & $0.056\pm0.014$ & $101.04\pm11.98$ & $214.64\pm3.04$ \\
LinearAttention & $0.224\pm0.000$ & 1.000 & $0.000\pm0.000$ & $0.031\pm0.010$ & $6.93\pm1.35$ & $445.91\pm3.03$ \\
\shaderow Performer & $0.188\pm0.000$ & 1.000 & $0.000\pm0.000$ & $0.036\pm0.011$ & $105.68\pm29.92$ & $953.26\pm2.53$ \\
Transformer & $0.235\pm0.000$ & 0.875 & $0.005\pm0.014$ & $0.040\pm0.010$ & $12.08\pm3.23$ & $335.90\pm2.44$ \\
\shaderow RWKV & $0.226\pm0.000$ & 0.625 & $0.134\pm0.197$ & $0.039\pm0.006$ & $26.05\pm15.14$ & $253.61\pm3.43$ \\
\bottomrule
\end{tabular}
\end{table}